\documentclass[a4paper,12pt,reqno]{amsart}
\usepackage{amssymb}

\usepackage{graphicx}

\usepackage{enumerate}

\usepackage{ifthen}

\usepackage{centernot}
\usepackage{mathtools}
\usepackage{ stmaryrd }

\usepackage{graphicx}

\nonstopmode\numberwithin{equation}{section}

\usepackage[applemac]{inputenc}

\newtheorem{theorem}{Theorem}[section]

\newtheorem{lemma}{Lemma}[section]

\newcommand{\be}{\begin{equation}}
\newcommand{\ee}{\end{equation}}

\newcommand{\R}{\mathbb{ R}}

\allowdisplaybreaks
    
\begin{document}
\title{  On the overfly algorithm in  deep learning of neural networks     }

\author{
Alexei Tsygvintsev
}
\address{
U.M.P.A, Ecole Normale Sup\'{e}rieure de Lyon\\
46, all\'{e}e d'Italie, F69364 Lyon Cedex 07
}
\email{atsygvin@umpa.ens-lyon.fr}

\bigskip
\begin{abstract}
In  this  paper  we  investigate  the supervised backpropagation   training  of multilayer  neural networks  from  a  dynamical  systems point of view.  We discuss  some links with the qualitative  theory of 
differential equations  and  introduce   the overfly algorithm   to tackle  the local minima problem.
Our approach is   based on the existence of  first integrals of the generalised gradient system with build--in dissipation.

    \end{abstract}.

\maketitle
\pagestyle{myheadings}

\section{Introduction. The dynamics of gradient flow. Neural networks and backpropagation.}
  Let $ F \, : \,  U \to \R$ be a smooth function in some open domain $U\subset \R^n$.   We equip  $U$ with the topology induced by the standard  Euclidean norm $||\cdot||$   defined by the canonical scalar product 
  $<x,y>=\sum x_i y_i$.
 The gradient vector field  defined  in $U$ by $F$  is  given by $V(x)=-\nabla F= -(\frac{\partial F}{\partial x_1}, \dots, \frac{\partial F}{\partial x_n})^T$, where $x=(x_1,\dots,x_n)^T$ are canonical coordinates in $U$.
 The critical points of $F$ are the solutions of $V(x)=0$, $x\in U$. Let $K$ be the set of all critical points of $F$ in $U$ (which can be unbounded and/or contain  non--isolated points). 
 
  The following theorem   \cite{Hirsch}, \cite{Wiggins}  is a classical result describing the asymptotic behaviour of solutions of the  gradient differential  system: 
 \begin{equation} \label{1}
 x'=V(x), \quad x\in U\,.
 \end{equation}
 \begin{theorem}
 Let $x_0\in U$ be the initial condition of \eqref{1}. Then every  solution $t\mapsto x(t)$, $x(0)=x_0$ either leaves all compact subsets of $U$   or approaches as $t \to +\infty $  the critical set  $K$ i.e
 \begin{equation}
  \lim_{t\to +\infty}\,  \inf_{y \in K} \, || x(t)-y||=0\,.
\end{equation}
\end{theorem}

In particular, at regular points,  the trajectories of \eqref{1} cross the level surfaces of $F$  orthogonally  and isolated minima of $F$ (which is a {\it Lyapunov} function \cite{Salle} ) are asymptotically equilibrium points.

Under the additional  analyticity condition the  above convergence result can be made stronger:
\begin{theorem}(Absila,Kurdyka, \cite{Absila})
 Let $F$ be real analytic in $U$. Then  $y\in K$ is a local minimum of $F$ iff it is asymptotically stable equilibrium point of \eqref{1}.
\end{theorem}

It should be noticed that the  gradient system \eqref{1} can not have any  non--constant periodic or recurrent solutions, homoclinic orbits or heteroclitic cycles. Thus, trajectories of gradient dynamical systems have quite  simple asymptotic behaviour.

Nevertheless, the localisation of basin of attraction of any equilibrium  point (stable or saddle one)  belonging to $K$  is a non trivial problem. 

Supervised machine  learning in multi--layered neural networks can be considered as application of gradient descent method  in a non--convex  optimization problem. The corresponding cost (or error) functions   are of the general form
\begin{equation}
E=\frac{1}{2} \, \sum (p_i-f(W,A^i)^2\,,
\end{equation}
with data set $(A^i,p_i)$ and  a certain  highly non--linear function  $f$  containing the  weights $W$.  
The main problem of the machine learning is to  minimize the cost function $E$  with a suitable choice of  weights $W$.  
A gradient method, described above and called {\it backpropagation} in the context of neural network training,  can get   stuck  in local minima or take very long time to run in order to  optimize $E$. 
This is due to the fact that general properties of the cost surface are usually unknown and  only the  trial and error numerical methods are available  (see  \cite {JL},  \cite{MG}, \cite{SK}, \cite{FY}, \cite{ED}, \cite{NN}, \cite{BK})).  No theoretical approach is known to provide the exact  initial weights in  backpropagation with guaranteed convergence to the global minima of $E$.
One of most powerful techniques used in backpropagation   is the {\it adaptive learning rate selection} \cite{CC}  where  the  step size of iterations  is gradually raised in order to escape a local minimum.
Another approach is based on  {\it random initialization}  \cite{PA} of weights in order to fortunately select them to be close to the values that give the global  minimum of the cost function. The deterministic approach, called {\it global descent},  was proposed in \cite{CB}  where  optimization was formulated  in terms of the flow of a special deterministic dynamical system.

The present work  seeks to integrate the ideas from the  theory of ordinary differential equations  to enrich the theoretical framework and  assist in better understanding the nature of convergence  in the training of  multi--layered neural networks.  The principal contribution is to propose the natural extension of classical gradient descent method by adding new degrees of freedom  and reformulating the problem in the new extended phase space of higher dimension.  We argue that this brings a deeper insight into the convergence  problem since  new equation become simpler algebraically and admit a family of  known  first integrals.  While this proposal may seem radical, we believe that it offers a number of advantages on both theoretical and as numerical  levels as our experiments clearly show.
Common sense suggests that embedding the dynamics of a gradient flow in a more general phase space of a new more general dynamical system is always advantageous since it can bring new possibilities to improve the convergence and escape local minima by  embedding  the cost surface into the higher dimensional phase space. 

The study   is divided into three  parts.
In Section 2 we begin by  reminding how the gradient descent method is  applied  to train  the simplest possible neural network with only output layer. That corresponds to  the conventional backpropagation algorithm  known for its simplicity and which is frequently used in deep learning. Next we  introduce   a natural extension of the gradient system which is done by replacing the weights of individual neurones within the output  layer by  their nonlinear outputs.
That brings more complexity to the iterative method, since the number of parameters  rises considerably, but at the same time,  the training data becomes built up into  network in a quite natural way. The so obtained generalised gradient system is later converted to the observer one (see \cite{Busvelle}). The aim is to turn the  constant level of   known first integrals into the attractor set. We will explain   how the Euler iterative method, applied to the observer system, and called overfly algorithm,  is involved in achieving of convergence to the global minimum of the cost function.
 Sections 3 and 4  discuss the  applications of this algorithm in training of $1$--layer and multilayer networks. The objective is  to put forward an explanation of how to  
 expand the backpropagation algorithm to its overfly version  via  modifying the weights updating procedure only for the first network's layer.  In Section 5 we provide  concrete numerical examples to illustrate the efficacy of the overfly algorithm in training of some particular neural networks.

\section{Neural network without  hidden layers}
In this section we give an elementary  algebraic description of the simplest no hidden layer neural network  called also a {\it perceptron} (see \cite{Gallant}).

We define  the {\it sigmoid} function
\begin{equation} \label{s}
\sigma(t)=\frac{1}{1+e^{-t}}, \quad  t\in \R\,,
\end{equation}
  as a particular solution of the {\it logistic} algebraic differential equation:
\begin{equation}
\sigma'(t)=\sigma(t)(1-\sigma(t))\,.
\end{equation}
In particular, $\sigma\, : \, \R \to (0,1)$ is  increasing and  rapidly convergent map  as $t\to \pm \infty$.

Let  $X\in\R^n$ and  $A\in \R^n$ be  two vectors called respectively {\it weight} and {\it input} ones . 
The analytic  map  $f\, :\, \R^n\to (0,1)$ defined by 
\begin{equation} \label{2}
f_X\, : \, A \mapsto \sigma(<A,X>) \, ,
\end{equation}
is called  a no hidden layer neural network.

Let 
\begin{equation} \label{ty}
(A^i, p_i), \quad 1\leq i \leq N\,,
\end{equation}
be the {\it  training set }  of \eqref{2} containing  $N$ input data vectors $A^i\in \R^n$ and  corresponding  scalar output values $p_i\in (0,1)$.
We want to determine the weight vector $X$ so that the $N$ values $f_X(A^i)$ match outputs $p_i$ as better as possible.
That can be achieved by minimising the so called {\it cost}  function
\begin{equation} \label{cost}
E(X)=\frac{1}{2}\, \sum_{k=1}^N\, (p_k- f_{X}(A ^k))^2\,,
\end{equation}
or, after the substitution of \eqref{2}:
\begin{equation}
E(X)= \frac{1}{2}\, \sum_{k=1}^N\, (p_k-\sigma(<A^k,X>))^2\,.
\end{equation}
In general,  $E \, : \R^n\, \to (0,1)$ is not coercive and not necessarily convex map. 

To apply the gradient descent method one considers  the following  system of   differential equations
\begin{equation} \label{3}
X'=-\nabla E(X)\,.
\end{equation} 
Since $E$  is always decreasing  along the trajectories of  \eqref{3}, it is natural to solve it   starting from some  initial point $X_0\in \R^n$ and 
use  $X(t), X(0)=X_0$  to minimise $E$.  The solution $X$  can converge (in the ideal case)  to the global minimum of $E$ or, in the less favourable case, $||X(t)||\to +\infty$ or  $X$  converges to  local minima or saddle points.

The {\it backpropagation}  method  \cite{Gallant}  for a   neural network   can be viewed as the Euler  numerical method \cite{Hairer} of solving of  a gradient system \eqref{3}.

Here one approximates the time derivative  by its discrete version
\begin{equation} \label{Eu}
X' (t)  \approx \frac{X(t+h)-X(t)}{h}\,,
\end{equation}
for some small step $h>0$ so that the approximative solution of \eqref{3}  $\bar X_k \approx X(t_k)$ at time $t_k=k h$ can be  obtained by iterations:
\begin{equation} \label{Euler}
\bar X_{k+1}=\bar X_k-h\nabla E(\bar X_k), \quad \bar X_0=X_0, \quad  k\geq 0\,.
\end{equation}

 We write \eqref{3} in a more simple algebraic form   by  introducing
the additional variables
\begin{equation}
M_i=\sigma(<A^i,X>), \quad i=1,\dots, N\,,
\end{equation}
representing  the nonlinear outputs of the network for $N$  given inputs  $A^i$ of the  training set.
Using the equations \eqref{3} to compute the derivatives  $M_i'$,  one obtains  the following system of $N$ differential equations 
\begin{equation} \label{44}
{M' _i} =M_i(1-M_i)\,\sum_{j=1}^N\, (p_j-M_j)M_j(1-M_j)G_{i,j}\,,
\end{equation}
with $G=G_{i,j}=<A^i,A^j>$ -- the $N\times N$ symmetric Gram matrix. We call \eqref{44} the {\it generalised gradient system}.

Let $D$ be $n \times N$ matrix defined by  $D=(A^1,\dots, A^N)$. Then $G=D^T D$ and, 
as  known from the elementary linear algebra:  $\mathrm{rank} (G)=\mathrm{rank} (D)$ and $\mathrm{Ker(G)}=\mathrm{Ker(D)}$.
Since the number of training vectors $N$  usually exceeds the total number of  weights  $n$ of the  network, we can assume that $N>n$.

Thus, since $\mathrm{rank} (G)\leq n$,  we have  $\mathrm{dim}(\mathrm{Ker}(G))\geq N-n>0$. 

Let $C=(C_1,\dots,C_N)^T\in \mathrm{Ker}(G)$ be a non--zero vector from the null space of $G$ and $I_N=(0,1)^N=(0,1)\times \cdots \times (0,1)\subset \R^N$. 
As seen from the equations \eqref{44}, $I_N$ is invariant under the flow of the system. Indeed, $M_i=0$ and $M_i=1$ are invariant hypersurfaces. 

\begin{theorem} \label{th1}
The function
\begin{equation} \label{66}
I_C=\sum_{k=1}^N \, C_k \,  {ln}\left( \frac{ M_k}{1-M_k}  \right), \quad M=(M_1,\dots,M_N)^T\in I_N\,,
\end{equation}
is a real analytic first integral of  the system \eqref{44}.

There exists   $p=N-\mathrm{dim}(\mathrm{Ker}(D))>0$ functionally independent first  integrals of the above form.
\end{theorem}

\begin{proof}
The first statement can be checked straightforwardly by derivation of \eqref{66}  using \eqref{44}.  We notice that if $ 0<M_i<1$ then $ M_i/(1-M_i)>0$. Thus, one has the real analyticity property of $I_C$. 
The linearity and functional independency  of $I_C$, $C\in \mathrm{Ker}( D)$  follow directly from the definition  \eqref{66}.
\end{proof}

In the rest of the paper we will always assume that $\mathrm{rank}(D)=n$ i.e  the set $D$ contains  sufficiently many   independent vectors.

Let $C^1, \dots, C^p$, $p=N-n$  be the basis of  $\mathrm{Ker}(D)$. Using the vector  notation 
\begin{equation} \label{eq}
F(M)=\left( {ln}\left( \frac{ M_1}{1-M_1}  \right), \dots,  {ln}\left( \frac{ M_N}{1-M_N}  \right)\right)^T, \quad M=(M_1, \dots,M_N)^T\,,
\end{equation}
 the family of the first integrals given by Theorem \ref{th1} can be  written simply as
 \begin{equation} \label{forms}
   I_{C ^i}(M) = <C^i, F(M)>,  \quad  i=1,\dots,p\,.
  \end{equation}
  
   Let $H \, : I_N \to \R^p$ , $I_N=(0,1)^N=(0,1)\times\cdots\times(0,1) \subset \R^N$ be the map defined by
 \begin{equation}\label{fg2}
 H(M)=(I_{C ^1}(M),\dots, I_{C^p}(M))^T\,.
  \end{equation}

\begin{lemma}  
$H  \, : I_N \,  \to \R^p$ is a submersion.
\end{lemma}
\begin{proof}
This follows directly from the fact that $C^1, \dots, C^p$ are linearly independent vectors and  \eqref{forms}.
\end{proof}
Thus, for all $y\in \R^p$ the set $ \Gamma_y=I_N\cap H^{-1}(y)$ is a $n$--dimensional invariant  manifold for the system \eqref{44}.

\begin{lemma} \label{DIFF} 
$\Gamma_0$ is diffeomorphic   to $\R^n$.
\end{lemma}

\begin{proof}
Let $X\in \R^n$.  We define  the map $ \Phi \, : \, \R^n \to  \R^N$ by
\begin{equation}
 \Phi(X)=(\sigma(<A^1,X>),\dots, \sigma(<A^N,X>)^T\,.
\end{equation}
Then, $I_{C^i}(\Phi(X))=\sum\limits_{j=1}^N C_{ij}<A^j,X>=< \sum\limits_{j=1}^N C_{ij}\, A^j,X>=0$ and so $\Phi \, : \, \R^n \to  \Gamma_0$.

To show that $\phi$ is invertible, let us fix $M\in \Gamma_0$.  Since $\sigma \,  : \,  \R \to (0,1)$ is one to one, there exists unique vector 
$Z=(Z_1, \dots, Z_N)^T\in\R^N$,   such that     $M_i=\sigma(Z_i)$, for $i=1, \dots, N$
and 
\begin{equation} \label{sys}
<C^i,Z>=0, \quad i=1, \dots, p\,,
\end{equation}
because  $F(M)=Z$ by substitution into  \eqref{eq}.

We are looking now  for the solution $X\in \R^n$ of the linear system  $<A^i,X>=Z_i$, $ i=1, \dots, N$ which can be written in the vector form as $A^T X=Z$ .  The linear map $\phi \, : \R^n \to \R^N$, $\phi(X)=A^TX$ has $\mathrm{rank}(\phi)=n$.  Moreover, $\mathrm{Im}(\phi)=\mathrm{Ker}( D ) ^{\perp}$ where orthogonality is defined by the scalar product $< , >$. Indeed, $\mathrm{Im}(\phi) \subset \mathrm{Ker}( D ) ^{\perp}$, by the direct verification,  and $\mathrm{dim}(\mathrm{Im}(\phi))= \mathrm{dim}(   \mathrm{Ker}( D ) ^{\perp}  )$ by the rank–nullity theorem. Hence, the map $\phi \, : \, \R^n \to  \mathrm{Ker}( D 
^{\perp})$ is a linear bijection and  the linear equation $A^T X=Z \iff   \phi(X)=Z$  admits the unique  solution $X$ since $Z\in \mathrm{Ker}( D ) ^{\perp}$ as follows from \eqref{sys}. The proof is done. 
\end{proof}

The system \eqref{44}  can be written in the vector form as $ M'=V(M)$ where $V$  is  a complete  in $I_N$ vector field ($I_N$ is a bounded  open invariant set).
 Let $\epsilon>0$  and 
\begin{equation} \label{UU}
U_{\epsilon}= \left \{  M  \in I_N \, :\, r(M)=||H(M)||   \leq \epsilon     \right  \}      \,, 
\end{equation}
be the $\epsilon$--neighbourhood of $\Gamma_0$.
Together with \eqref{44}, consider the following  {\it observer}  system 
\begin{equation} \label{observer}
M'=W(M)=V(M)+P(M), \quad M\in I_N\,,
\end{equation}
where
\begin{equation} \label{zzz}
P(M)=-k\Pi(M)\tilde R F(M), \quad \tilde R= \Theta R^{-1} \Theta ^t , \quad R=\Theta^t\Theta\,.
\end{equation}
.
Here,  $\Theta=(C^1, \dots, C^p)$, $\Theta\in M_{p,N}(\R)$ and 
\begin{equation}
\Pi(M)=\mathrm{diag}\left (M_1(1-M_1),\dots,M_N(1-M_N)\right)\,.
\end{equation}

The matrix $R$ is  invertible and positive definite   since $\mathrm{rank}(\Theta)=N-n$ . Thus,   the  vector field $P$  is well defined in $I_N$.

\begin{theorem}  \label{main}
Let $M_0\in I_N$ and $t \to M(t)$ be the solution of the observer system \eqref{observer} with the initial condition $M(0)=M_0$.
Then
\begin{equation} \label{rea}
 r(M(t))=r(M_0)e^{-kt}, \quad t\geq 0\,,
 \end{equation}
 with $r$ defined in \eqref{UU}. In particular $\lim\limits _{t\to + \infty} r(M(t))=0$ and $U_{\epsilon}$ is invariant set containing  $\Gamma_0$ as  attractor.
\end{theorem}
\begin{proof}
 Firstly,  we write the $H$ introduced in \eqref{fg2}  in the compact matrix form
$$
H(M)=\Theta^t F(M)\,.
$$
We follow now the idea of the proof of Main Lemma from \cite{Busvelle}, p. 377. 
and  derive $r^2$ with respect to time along the solution of \eqref{observer} to obtain a simple differential equation:
\begin{equation} 
\frac{d r^2(M(t))}{dt}=-2kr^2(M(t)), \quad r^2(M(0))=r(M_0)\,,
\end{equation}
which can be easily solved to get  \eqref{rea}. 
\end{proof}
 We notice that our choice  of the term $P$ in \eqref{observer} is different from one proposed in  \cite{Busvelle}.
\begin{lemma} \label{lemma}
The function 
\begin{equation}
E(M)=\frac{1}{2}\,  \sum_{i=1}^N\, (p_i-M_i)^2\,,
\end{equation}
 is a Lyapunov one   and verifies $\displaystyle \frac{d E(M(t))}{dt}\leq 0$ for every solution $t\mapsto M(t)$, $M(0)\in I_N$ of \eqref{44}.
\end{lemma}
\begin{proof}
It is sufficient   to derive $L$ and to  use  the positiveness of the Gram matrix $G=D^T D$.
\end{proof} 

Now we shall explain the role of the observer system \eqref{observer}  in the problem of  minimisation  of the cost function  \eqref{cost}.

Firstly, while using  the standard gradient descent method,  instead of dealing with  the system \eqref{3},  one can solve  the observer equations  \eqref{observer} with some  
initial condition $M(0)\in \Gamma_0$ and use then Lemma \ref{DIFF} to compute $X$ as corresponding to  $M(t)$ for some sufficiently large $t>0$.  It is well known that applying  the  Euler method \eqref{Eu} to solve \eqref{3},  i.e  following   the conventional  backpropagation algorithm, leads to accumulation of a global error proportional to the step size $h$.  At the same time,  the numerical integration  of the  observer system \eqref{observer}, as due to the existence of the attractor set $\Gamma_0$,  is much more stable numerically  since the solution is  attracted by the integral manifold  $\Gamma_0$ (see \cite{Busvelle} for more details and examples).

Second improvement  brought by the observer system \eqref{observer} is more promising. Imagine we start integration of \eqref{observer}  with  the perturbed initial condition  $M(0) \in U_{ \epsilon}$, $M(0)\not \in \Gamma_0$ for some $\epsilon>0$. Then, according to Theorem \ref{main}, $M(t)\to \Gamma_0$, $t\to +\infty$   and as follows from Lemma \ref{lemma}, $t \mapsto E(M(t))$ will be decreasing function of $t>0$ in a neighbourhood of   $\Gamma_0$ since $P=0$ on $\Gamma_0$.   That can be seen as  a coexistence of the  local dynamics  of the observer system in $U_{\epsilon}$,   pushing $M$ to the equilibrium point $M_i=p_i$, $i=1,\dots,N$ of \eqref{3} and the  dynamics of the gradient system \eqref{3} on $\Gamma_0$ forcing  $M$ to approach
the critical points set (see Figure 3).

One can suggest that this  kind of double dynamics  increases considerably  the chances of convergence to the global minimum of the cost function \eqref{cost}.  We call {\it overfly} the training of the neural network \eqref{2} done  by solving the observer system \eqref{observer}  with help of the  Euler first--order  method  starting from some initial point  $M(0)\in U_{\epsilon} \setminus \Gamma_0$.



\section{The $1$--hidden layer network case }

In this section we describe the generalised gradient system of differential equations appearing in the supervised  backpropagation training of a $1$--hidden layer network.
As in the previous section, let $A\in \R^n$ belongs to the training set  \eqref{ty}.  Let   $Y^1,\dots,Y^m\in \R^n $ be  $m$ weight vectors of the hidden layer and $X\in \R^m$ is the weight vector of the output layer.

The  $1$--hidden layer neural network is a real analytic map $ f_{Y,X} \, : \,  \R^n \to (0,1)$ defined as follows 
\begin{equation} \label{UY}
f_{Y,X}(A)=\sigma(<\pi_Y(A),X>)\, ,
\end{equation}
where $\pi_{Y}(A)=(\sigma(<A,Y^1>),\dots,\sigma(<A,Y^m> ))^T$ are the  outputs of the first layer. 
We want to minimise the same cost function 
\begin{equation} \label{costy}
E(Y,X)=\frac{1}{2} \, \sum _{i=1}^N\, (p_i-f_{Y,X}(A^i))^2\,,
\end{equation}
where $(A^i, p_i)$, $i=1, \dots, N$ is the training set.
To solve the optimisation problem one can  define   the gradient system analogous to \eqref{3} with respect to the  vector variables $Y^i$ and $X$:
\begin{equation} \label{layer1}
{Y^i}'=-\nabla_{Y^i}E, \quad X'=-\nabla_X E, \quad 1\leq i \leq m\,. 
\end{equation}

 Let us introduce the  following  scalar variables:
\begin{equation}\label{era}
 \Omega_{jk}=\sigma(<A^j,Y^k>) \,.
\end{equation}
The function \eqref{costy}, expressed in new variables, takes the following form
\begin{equation} \label{newcost}
E(\Omega,X)=\frac{1}{2}\sum_{i=1}^N\, (p_i-\sigma(<\Omega^i,X>))^2, \quad   \Omega^i=(\Omega_{i1},\dots,\Omega_{im})^T\,.
\end{equation}

The  differential equations describing the generalised gradient system for the neural network \eqref{UY} are obtained  by derivation of \eqref{era} with help of \eqref{layer1}:
\begin{equation}\label {EQUATIONS}
\left \{
\begin{array}{lllll}  
\Omega_{ik}'=m_{ik}(\Omega, X)=\Omega_{ik}(1-\Omega_{ik})X_k\, \sum\limits _{j=1}^N\, (p_j-\omega_j) \omega_j(1-\omega_j)\Omega_{jk}(1-\Omega_{jk})G_{ij}\,, \\ 
X'=-\nabla_X E=\sum\limits_{i=1}^N\, (p_i-\omega_i)\omega_i (1-\omega_i) \Omega^i, \quad \omega_i=\sigma(<\Omega^i,X>)\,,
\end{array}
\right.
\end{equation}
where $G_{ij}=<A^i,A^j>$ is the Gram matrix defined by the training set \eqref{ty}.

The next  theorem is a generalisation of Theorem  \ref{th1}.  Let $r=\mathrm{dim}(\mathrm{Ker}(G))$ and  $ \mathrm{Ker}(G)= \mathrm{Span}(C^1,\dots,C^r)$, $C^j=(C_{j1},\dots,C_{jr})^T$.

\begin{theorem} \label{rrr}
The generalised gradient system $\eqref{EQUATIONS}$ admits $rm$ functionally independent first integrals
\begin{equation} \label{rtf}
I_{C^j,k}(\Omega)=\sum_{i=1}^N \, C_{ji}\, ln \left( \frac{\Omega_{ik}}{1-\Omega_{ik}} \right), \quad   1\leq j \leq r, \quad  1\leq  k  \leq m\,.
\end{equation}
The cost function $E$ defined by \eqref{newcost} is a Lyapunov function for $\eqref{EQUATIONS}$
\end{theorem}
\begin{proof}  One verifies directly that  $I_{C^j,k}$ is a first integral of  $\eqref{EQUATIONS}$ by simple derivation.  A rather  tedious but elementary calculation shows that $E( \Omega(t),X(t))' \leq 0$ along the solutions of 
 $\eqref{EQUATIONS}$ (see also Theorem \ref{MAINT} for the general proof).
\end{proof}

The observer system, analogous to \eqref{observer}, written for the generalised gradient system \eqref{EQUATIONS},  can be obtained straightforwardly by replacing the first equation of \eqref{EQUATIONS} with
\begin{equation} \label{observer1}
\Omega'=U(\Omega,X)+P(\Omega), \quad X'=-\nabla_XE,  \quad     1\leq i \leq N, \quad 1\leq k \leq m\,,
\end{equation}
where the additional  term $P$  is defined in similar to \eqref{zzz} way with help of the  first integrals defined by Theorem \ref{rrr}.

Indeed, let $K=(K_{ij} )_{1\leq i \leq N, 1 \leq j \leq m}$  and $S=(S_{ij} )_{1\leq i \leq N, 1 \leq j \leq m}$  are  two matrices defined by 
\begin{equation}
K_{ij}=\Omega_{ij}(1-\Omega_{ij}) , \quad S_{ij}=ln \left (  \frac{\Omega_{ij}}{1-\Omega_{ij}}       \right)\,.
\end{equation}

To prove  the result similar to Theorem \ref{main} one can define $P$ in  \eqref{observer1} as follows 
\begin{equation} \label{disp}
P=-kK\circ(\tilde R S)\,, 
\end{equation}
where the constant matrix  $\tilde R$ is the  same as in \eqref{zzz} and “$\circ$”  is the Kronecker matrix product. 
 
 Indeed, the first integrals defined  by  \eqref{rtf}  can be  written in a matrix  form:  $H(\Omega)=\Theta^t  S(\Omega)$. Then, deriving $r^2(\Omega(t))=|| H (\Omega(t)) ||_2^2$, where  $||\cdot ||_2$ is the Frobenius  matrix norm, 
  along a solution  $t\mapsto \Omega(t)$  of \eqref{observer1}, one gets  
 \begin{equation} \label{r2}
\frac{ d r^2(  \Omega(t)  )  }  {dt}=-2kr^2(\Omega(t)) \,,
\end{equation}
and  so 
\begin{equation} 
 r(\Omega(t))=r(\Omega_0)e^{-kt}, \quad t\geq 0\,.
 \end{equation}

The practical implementation of the overfly algorithm in the $1$--layer case is analogous to one described in Section 2.  Instead of modifying the weights of the first layer $Y^i$ at every step of the gradient descent, one updates the values of $\Omega_{ik}$ and $X$ applying  the Euler method   to solve the observer  equations \eqref{observer1}. 

For the sake of simplicity, we will provide below the explicit matrix form of the system \eqref{observer1} which is better  adopted to numerical implementations. 
We introduce the following diagonal matrices:
\begin{equation}
\begin{array}{lll}
\hat P_{\omega}=\mathrm{diag}((p_1-\omega_1)\omega_1(1-\omega_1),\dots, (p_N-\omega_N)\omega_N(1-\omega_N))\,,\\ 
\hat X=\mathrm{diag}( X_1, \dots, X_m )\,,
\end{array}
\end{equation} 
and the $N$--vector
\begin{equation}
P_{ \omega }=( (p_1-\omega_1)\omega_1(1-\omega_1),\dots, (p_N-\omega_N)\omega_N(1-\omega_N) )^T\,.
\end{equation}
Let $X=(X_1, \dots, X_m)^t$ be the $m$--vector  of the output layer.
The observer system   \eqref{observer1}  can be written in the following compact form  
\begin{equation} \label{aze}
\left\{
\begin{array}{lll}
\Omega'=K \circ ( G {\hat P} _ {\omega}  K  {\hat X} - k\tilde R S ) \\
X'=\Omega^T P_{\omega}\,,
\end{array}
\right.
\end{equation}
where $K=\Omega-\Omega\circ \Omega$.

\section{General multilayer case}

We want to analyse a general  multilayer neuronal network with the architecture  $n-l-\cdots-1$. Here  $n$ is a number of inputs and $l$ is the number of neurones in the very first layer. The  network has only one output 
and in  every layer the  same sigmoid  function  \eqref{s}  is used.  The training set is defined by \eqref{ty}.
Let $Y^i\in \R^n$, $1\leq i \leq l$ be the weight vectors of $l$ neurones  of  the first  layer. We note $Z$ the weights of other network's layers. Let $A\in \R^n$ be the input vector. The generic multilayer  neural network can be written as the composition of  two maps:
\begin{equation} \label{maps}
f_{Y, Z}(A)=\Phi_Z \circ \pi_Y(A) \,,
\end{equation}
where  $\Phi_Z \, : \,  \R^l \to (0,1)$, $\pi=(\pi_1,\dots,\pi_l)^T\xmapsto{\Phi_Z} \Phi_Z(\pi)$  is defined jointly by all  layers different from the first one   and
\begin{equation} \label{ez}
 \pi_Y(A)=(\sigma(<A,Y ^1>), \dots, \sigma(<A, Y^l>))^T\,,
 \end{equation}
 is the output vector of the first layer.

Using the chain rule one obtains for every  $k=1, \dots,l$:
\begin{equation} \label{t1}
\frac{ \partial f_{Y, Z} }{ \partial Y_{ki}}= \left <\nabla  \Phi_Z,  \frac{\partial \pi_Y}{\partial Y_{ki}}\right>, \quad i=1, \dots,n\,,
\end{equation}
where, according to \eqref{ez},  
\begin{equation} \label{t2}
 \frac{\partial \pi_Y}{\partial Y_{ki}}=\sigma(<A,Y ^k>) (1- \sigma(<A,Y ^k>))  (\underbrace{0,\dots,A_i}_k, \dots,0)^T\,.
\end{equation}
Thus, combining together \eqref{t1},\eqref{t2} we obtain:
\begin{equation}
\frac{ \partial f_{Y, Z} }{ \partial Y_{ki}}=\sigma(<A,Y ^k>) (1- \sigma(<A,Y ^k>)) A_i   \frac {\partial{ \Phi_Z}} {\partial \pi_k}\,.
\end{equation}
We can compute now  the partial derivatives of the  cost function
\begin{equation} \label{costy1}
E(Y,Z)=\frac{1}{2} \, \sum _{j=1}^N\, (p_j-f_{Y,Z}(A^j))^2\,,
\end{equation}
with respect to the weights of the first layer:
\begin{equation} \label{qui}
\frac{ \partial E}{ \partial Y^k}=-\sum_{j=1}^N\, (p_j-f_{Y,Z}(A^j))\sigma(<A^j,Y ^k>) (1- \sigma(<A^j,Y ^k>))    \frac {   \partial  \Phi_Z     }       {   \partial \pi_k }      (   \pi_Y(A^j)  )     A^j\,.
\end{equation}
The  equation of the  gradient system corresponding to the weight vector $Y^k$ can be written as
\begin{equation} \label{v}
{Y^k}'=-\nabla_{Y^k} \,  E=-\frac{ \partial E}{ \partial Y^k}\,.
\end{equation}
Introducing the variables 
\begin{equation} \label{333}
\Omega_{pk}=\sigma(<A^p,Y ^k>)\,,
\end{equation}
called  the {\it splitting weights}, and whose derivatives can be found  with help of  \eqref{v}, we deduce from \eqref{qui} the following  differential equations
\begin{equation}
\Omega_{pk}' =\Omega_{pk}(1-\Omega_{pk}) \sum_{i=1}^N  \, (p_i-\Phi_{Z}(\Omega^i))\Omega_{ik} (1- \Omega_{ik} )    \frac {\partial{ \Phi_Z}} {\partial \pi_k}(  \Omega^i  )G_{ip}\,,
\end{equation}
where $\Omega^i=(\Omega_{i1}, \dots,\Omega_{il})^T$.

The above equations can be written also as
\begin{equation} \label{NNN}
\Omega_{pk}' = N_{pk}(\Omega,Z), \quad   1\leq p \leq N,  \quad 1\leq k \leq l\,.
\end{equation}
Indeed, $\displaystyle f_{Y,Z}(A^j)$ and $ \displaystyle \frac {   \partial \Phi_Z  } {    \partial \pi_k}       (\pi_Y(A^j))     $ are functions of $\Omega$ and $Z$ only. Moreover, the same holds for the cost function $E$ defined in \eqref{costy1} and its gradient $\nabla_Z E={\partial E} / { \partial Z}$: they can be written as functions  of variables  $\Omega$ and $Z$.

Let $r=\mathrm{dim}(\mathrm{Ker}(G))$ be the dimension of the null space 
of the Gram matrix  $G_{i,j}=<A_i,A_j>$  and $ \mathrm{Ker}(G)=\mathrm{Span}(C^1, \dots, C^r)$. We note $C^i=(C_{i1},\dots,C_{iN})^T$.

\begin{theorem} \label{MAINT}
Let 
\begin{equation}\label{gen2}
\Omega'=N(\Omega,Z), \quad  Z'=-\nabla_Z E (\Omega,Z)\,,
\end{equation}
be the generalised gradient system written  for  the multilayer network \eqref{maps}  with the training set $(A^ i,p_i)$, $i=1,\dots,N$.
Then \eqref{gen2} admits $rl$ independent first integrals of the form
\begin{equation} \label{fint}
 I_{C^j,k}(\Omega)=\sum_{i=1}^N \, C_{ji}\, ln\left(\frac{ \Omega_{ik}}{ 1-\Omega_{ik}}\right), \quad \Omega_{ik}=\sigma(<A ^i,Y ^k>)\,.
\end{equation}

The cost function \eqref{costy1} $E=E(\Omega,Z)$  is a Lyapunov function for  \eqref{gen2}.
\end{theorem}

\begin{proof}
It is straightforward to verify that $ I_{C^j,k}$ are functionally  independent first integrals of \eqref{gen2}. Accordingly to \eqref{maps}, \eqref{ez} and \eqref{333}, the cost function \eqref{costy1}, written in variables $\Omega$, $Z$,  is given by 
\begin{equation} \label{ncost}
E(\Omega,Z)=\frac{1}{2}\sum_{i=1}^N\, ( p_i-\Phi_Z(\Omega^i))^2, \quad \Omega^i=(\Omega_{i1},\dots,\Omega_{il})^T\,,
\end{equation}
in view of   \eqref{t1},\eqref{ez} and \eqref{333}. Let $ t \mapsto ( \Omega(t), Z(t))$ be a solution of \eqref{gen2}. Then
\begin{equation} \label{derv}
\frac{d}{dt} E(\Omega(t),Z(t))= \left <   \frac{\partial E}{\partial \Omega}, N          \right >_{\Omega}- \left <   \frac{\partial E}{\partial Z},   \nabla_Z E       \right   >_{Z}=  \left <   \frac{\partial E}{\partial \Omega}, N          \right >_{\Omega}-||   \nabla_Z E   ||_Z^2\,,
\end{equation}
where $<, >_{\Omega}$, $<,>_Z$ are the standard scalar products defined respectively  in spaces $\R^a$  and $\R^b$ where $a=pl$ is the total number of splitting weights $\Omega_{pk}$ and $b$ is the total number  of weights $Z$ of the neural network \eqref{maps}.  One writes with help of \eqref{NNN}:
\begin{equation}
 \left <  \displaystyle   \frac{\partial E}{\partial \Omega}, N          \right >_{\Omega}= -\displaystyle\sum\limits _{i=1}^N \, \displaystyle\sum\limits_{k=1}^l\, \frac{\partial E}{\partial \Omega_{ik}}\, N_{ik}= 
 -\displaystyle \sum\limits_{k=1}^l \left (\displaystyle \sum\limits_{i=1}^N  T_{ik}  \displaystyle \sum\limits_{j=1}^N 
 G_{ij}T_{jk}\right) \,,
 \end{equation}
 where $T_{ik}= (p_i-\Phi_Z(\Omega^i))\displaystyle\frac{\partial \Phi_Z}{\partial \pi_k}(\Omega^i) \Omega_{ik}(1-\Omega_{ik})$. 
Since  $G_{ij}$ is a positive matrix, the last equality implies $ \left <  \frac{\partial E}{\partial \Omega}, N          \right >_{\Omega}\leq 0$. Together with \eqref{derv} this yields that $E$ is a Lyapunov function  of \eqref{gen2}.
\end{proof}

The observer system, defined by analogy with \eqref{observer} for  the generalised gradient system \eqref{gen2},  can be written in the following form
\begin{equation} \label{obs3}
\Omega'=N(\Omega,Z)+P(\Omega), \quad  Z'=-\nabla_Z E (\Omega,Z)\,,
\end{equation} 
where the vector field  $P$, called the  {\it dissipation} term, is defined by the first integrals \eqref{fint}  and given by the same formula \eqref{disp}.

The overfly algorithm for neural network training, already described in previous sections, can be easily adopted to the general multilayer case.
The only difference from the conventional backpropagation applied to the  network \eqref{maps},  consists in replacing the weights of the first layer $ Y_{ij}$ by the splitting weights  $\Omega_{pk} $, while keeping updating the weights $Z$ of other layers accordingly to the usual  bacpropagation algorithm.
At each iteration step, the evolution of parameters $\Omega_{pk}, Z$ is governed by the Euler discretisation of the observer system  \eqref{obs3}.

\section{Conclusion and numerical results}

In this section we compare the usual backpropagation and the overfly methods  for some particular  neural networks.  We start by  a simple  no hidden layer  case \eqref{2}.

We  put  $n=1$ and $X=x \in \R$.  Let $N=5$ and  the input input values  are defined by 
\begin{equation}\label{45}
T=[79/100, -9/20, 7/10, -9/50, -19/25]\,,
\end{equation}
with the  corresponding output  vector $p$:
\begin{equation} \label{46}
p=[-1/20, -21/25, -11/100, 61/100, -83/100]\,.
\end{equation}
The couple $(T,p)$  defines the training set \eqref{ty}.

Analysing the equation  $E'(x)=0$,  with $E$ defined in \eqref{cost},  one calculates,  with help of Maple's 10  RootFinding routine,  two local  minima $A$ and $B$ (see Figure1) of the cost function $E$  in points $x_{A}= 2.510$,   $E(x_{A})= 1.967$  and $x_{B}= 6.067$, $E(x_{B})= 1.966$ with $B$ being the global minimum of $E$.

The gradient system \eqref{3} was solved  using the Euler method \eqref{Euler} with $h=1$ with the initial point $x(0)=3$.
After $d=3000$ iterations one obtains  $x=x_{d}=2.510$ with $ E(x_{d})= 1.967$ and   the  backpropagation  network converges  to the local minimum $A$.

To calculate the vector $M(0)$, corresponding to $x(0)$, one can apply Lemma \ref{DIFF} to find
\begin{equation}
M(0)=[0.879, 0.244, 0.853, 0.389, 0.129]^T
\end{equation}

Now, following  the overfly approach,  we consider the observer system \eqref{observer} with $k=0.002$ and initial conditions $M(0)+\tilde{M}$ with the perturbation vector $\tilde{M}$ defined by 
 \begin{equation}
\tilde{M}=[0.01,0.01,0.01,0.01,0.01]^T\,.
\end{equation}
The Euler method, applied to    \eqref{observer}  with $h=1$  provides after $\delta=3000$ iterations   the value $x=\tilde{x}_{\delta}=6.085$ with $E(\tilde{x}_{\delta})=1.966$. 
Since,  $\tilde{x}_{\delta}$ is sufficiently  close to $x_B$ we conclude that the overfly network   converges  to the global minimum $B$  rather than to the local one $A$.   So,  the benefits of the overfly training  are immediately visible.

We have tested numerically  the overfly method for  a $4-2-1$  neural network \eqref{UY}.  It  has $4$ inputs and   $1$ hidden layer with $2$ neurones ($n=4,m=2$).
Both hidden and output layer have biases. 
The input data set has $N=10$ entries  arranged into the  following  $4\times10$ matrix $A=[A^1\,\dots,A^{10}]$ :
\begin{equation} \label{a1}
{\tiny A=\left[ \begin {array}{cccccccccc}  0.234&-
 0.316&- 0.746& 0.064&
 0.124& 0.894&- 0.786&-
 0.076& 1.044&- 0.436
\\ \noalign{\medskip}- 0.385&- 0.835&
 0.015& 0.365&- 0.935&
 0.135& 0.335& 0.505&
 0.495& 0.305\\ \noalign{\medskip}
 0.764& 0.594& 0.684&-
 0.946& 0.024&- 0.196&-
 0.596& 0.534&- 0.436&-
 0.426\\ \noalign{\medskip}- 1.014&-
 0.074& 0.346& 0.876&-
 0.354&- 0.184&- 0.174&-
 0.254& 0.266& 0.566\end {array}
 \right] }
\end{equation}

The columns of $A$ were chosen randomly and have zero mean. The output target vector $p \in \mathbb{R}^{10}$ is of the form
{\small \begin{equation} \label{a2}
 [0.301, 0.30001, 0.30002, 0.30013, 0.30004, 0.30005, 0.30006, 0.30007, 0.30008, 0.30009]\,,
\end{equation} 
and corresponds to  a  highly deviated data set. In particular: 
\begin{equation} 
\frac{ p_1-p_2}{ p_3-p_2} =99  \quad \mathrm{and} \quad  \frac{p_6-p_5}{p_5-p_4}=1\,.
\end{equation}

Firstly, the standard  $4-2-1$ neural network \eqref{UY}  was trained on the above data set  using usual backpropagation method (BM) with  randomly chosen  in the interval $[-1,1]$ weights $Y$ and $X$. The number of iterations was $d=1500$ with the  step size  $h=0.1$.

Then, the overfly algorithm was applied, as described in Section $3$,  with randomly chosen   initial splitting weights $\Omega_{ij}\in(0,1)$, same $X$  and the dissipation parameter $k=0.01$. 
The observer system $\eqref{aze}$  was solved by Euler method with the same  step size  $h=0.1$ and using the same iteration number  $d=1500$. 
At each iteration  we  computed  the cost function value for both methods:   using the formula   \eqref{costy} for  BM and the expression \eqref{newcost} for the  overfly method (OM).
The final cost value, after $d$ iterations for BM, was $E_{BM}=0.588  \cdot 10^{-3}$ and  for OM it was $E_{OM}=3.499 \cdot 10^{-7}$  with the ratio   $E_{BM}/E_{OM} \approx 146$. 
Thus   the overfly algorithm significantly outperforms the conventional backpropagation for this particular  problem.  The Figure  2 contains graphs of both cost functions in  the logarithmic scale.
We notice that our example is quite generic  one since our numerical   experiments show that  statistically OM gives more precise results than BM  for the large deviation output data sets.

We notice that there is an obvious resemblance between  conventional backpropagation and overfly approaches. Below we summarise briefly the principal steps of the proposed method. \\

\noindent {\it Step 1: Splitting.}    Assuming that the training data $(A^i,p_i)$ is given, firstly,  it is necessary to compute   the generating vectors of the null--space of the matrix $D=(A^1,\dots,A^N)$ i.e 
determine $\mathrm{Ker}(G)$. Secondly,   one introduces $Nl$ splitting weights \eqref{333}  to replace  $nl$ weights of neurones of the first layer. In practice, the number $N$ of training examples can be considerably larger than the input size of the network $n$, so  the splitting brings more additional parameters to be stored in the memory.\\
\noindent {\it  Step 2: Dissipation.  }    Using the vectors spanning $\mathrm{Ker}(G)$  one  creates a procedure computing the dissipation term $P$  defined by \eqref{zzz}.
The matrix inversion in \eqref{zzz} can be done, in the beginning,  using the  conjugate gradient   algorithm \cite{M} i.e in an iterative way.
Indeed, the matrix $R$ is symmetric and positive definite. \\
\noindent {\it Step 3:   Generalised gradient -- observer:} The first--order  Euler iterative method  is applied next to solve the observer system \eqref{obs3}. The optimal choice of the step $h$ and the constant $k$ depends on the concrete problem. We suggest to run firstly the usual backpropagation (i.e choosing the initial value $\Omega\in \Gamma_0$)  and  try to improve the result using several choices of initial values for $\Omega_{ ik }\in (0,1)$ and of $k>0$ in the overfly training.  If $k=0$ i.e then no dissipation term is present and starting with $\Omega\not \in \Gamma_0$ the method can provide only the approximation of the neural network weights.  But it is still worth trying: if initial values of $\Omega$ are sufficiently close to $\Gamma_0$ they will stay near $\Gamma_0$ (first integrals \eqref{fint}  are conserved)   and the algorithm's complexity is greatly reduced since
no dissipation is added at every iteration (no need to compute $P$ in \eqref{obs3} at every step).  Thus, the  neural network can be trained in alternation with dissipation switched on and  off.  We notice as well that  the proposed method
can be easily  adopted to take into account  biases by introducing additional  bias nodes.

Clearly, further research and more numerical evidences  are necessary to confirm the benefits of  the  overfly algorithm.  
The results of our study suggest a number of new avenues for research and numerical experiments.\\ 

\noindent {\bf Acknowledgments.} The study was supported by the  PEPS project Sigmapad, Intelligence Artificielle et Apprentissage Automatique.

\begin{figure}\label{fig1} \centering \includegraphics[scale=0.5] {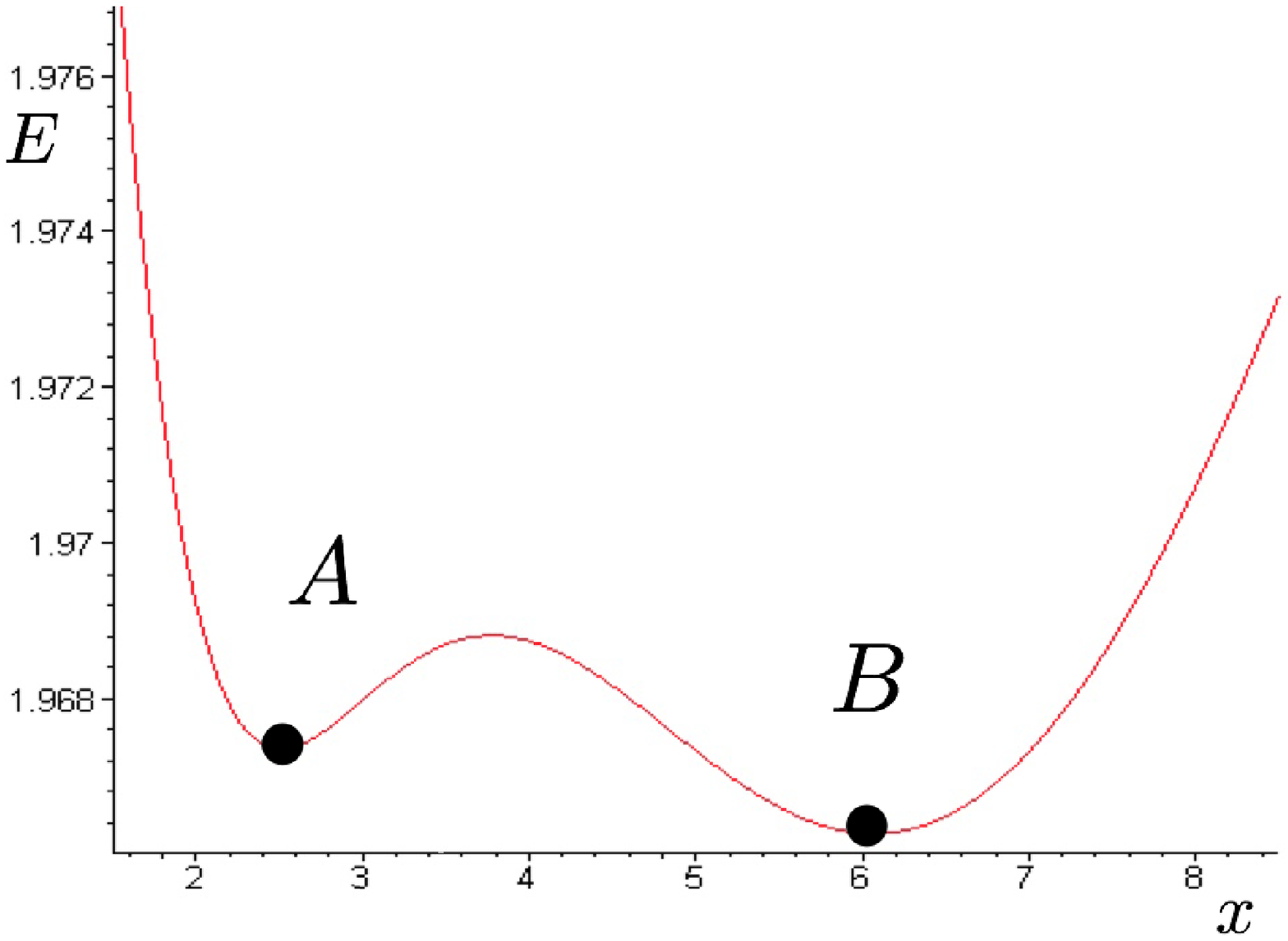} \caption{ The graph   of  the cost function $E$ for the training set \eqref{45}, \eqref{46} }
\end{figure}

\begin{figure}\label{fig33} \centering \includegraphics[scale=0.5] {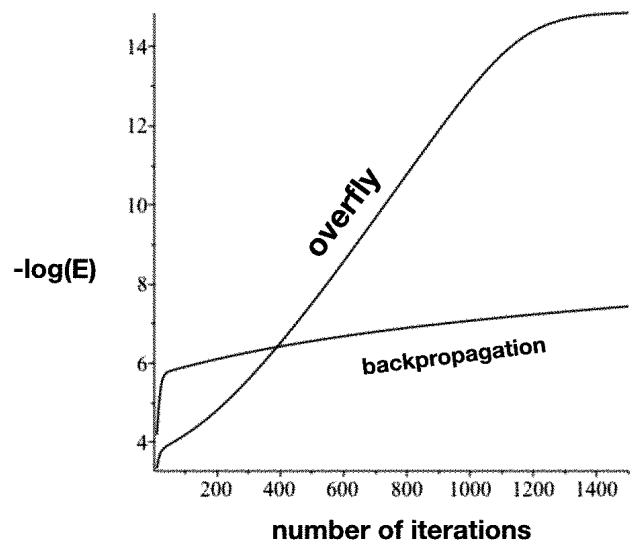} \caption{ $4-2-1$  neural network, testing performance of overfly and backpropagation for the data set  \eqref{a1}, \eqref{a2} }
\end{figure}

\begin{figure}\label{fig2} \centering \includegraphics[scale=0.5] {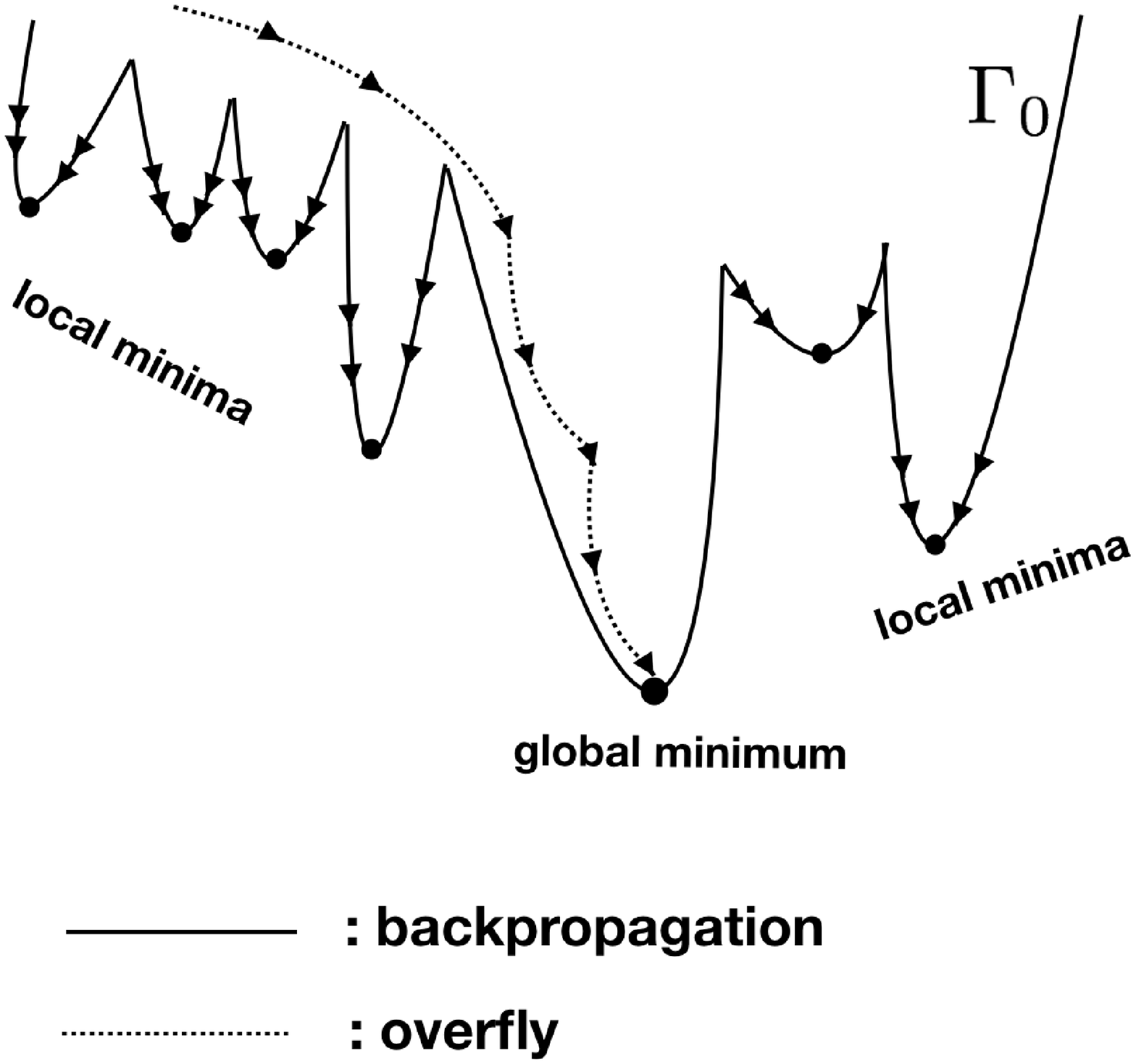} \caption{ }
\end{figure}

\end{document}